\documentclass[twoside]{article}

\usepackage[accepted]{aistats2019}
\usepackage{algorithm}
\usepackage{algpseudocode}
\usepackage{amsthm}
\usepackage{amsfonts}
\usepackage{amsmath}
\usepackage{amssymb}
\usepackage{blkarray}
\usepackage{booktabs}
\usepackage[capitalise]{cleveref}
\usepackage{caption}
\usepackage{graphicx}
\usepackage{subcaption}
\usepackage{xfrac}
\usepackage{epstopdf}
\usepackage{rotating}  

\newtheorem{theorem}{Theorem}[section]
\newtheorem{lemma}[theorem]{Lemma}
\newenvironment{remark}[1][Remark]{\begin{trivlist}
\item[\hskip \labelsep {\bfseries #1}]}{\end{trivlist}}

\DeclareMathOperator*{\argmin}{arg\,min}

\begin{document}

%

%

\twocolumn[

\aistatstitle{On the Interaction Effects Between Prediction and Clustering}

\aistatsauthor{ Matt Barnes 
\And Artur Dubrawski 
}

\aistatsaddress{ Carnegie Mellon University \And Carnegie Mellon University} 
  ]

\begin{abstract}
Machine learning systems increasingly depend on pipelines of multiple algorithms to provide high quality and well structured predictions. This paper argues interaction effects between clustering and prediction (e.g.\ classification, regression) algorithms can cause subtle adverse behaviors during cross-validation that may not be initially apparent. In particular, we focus on the problem of estimating the out-of-cluster (OOC) prediction loss given an approximate clustering with probabilistic error rate $p_0$. Traditional cross-validation techniques exhibit significant empirical bias in this setting, and  the few attempts to estimate and correct for these effects are intractable on larger datasets. Further, no previous work has been able to characterize the conditions under which these empirical effects occur, and if they do, what properties they have. We precisely answer these questions by providing theoretical properties which hold in various settings, and prove that expected out-of-cluster loss behavior rapidly decays with even minor clustering errors. Fortunately, we are able to leverage these same properties to construct hypothesis tests and scalable estimators necessary for correcting the problem. Empirical results on benchmark datasets validate our theoretical results and demonstrate how scaling techniques provide solutions to new classes of problems.
\end{abstract}

\section{Introduction}
With the increasing prevalence of machine learning solutions, there is a growing concern over the interactions between algorithms in complex systems \cite{Sculley2014}. Leveraging multiple learning algorithms is a common technique to optimize performance and incorporate structured prior knowledge. For example, most autonomous vehicles benefit from using separate models for perception of traffic lights, object detection and tracking, localization, predicting actor behavior and ultimately planning an optimal trajectory. Although attempting to directly map from visual inputs to control outputs is simpler, this approach is known to achieve inferior performance.
Breaking the larger problem into a sequence of smaller problems may be advantageous for many reasons, but it can create additional challenges which must be addressed. 

In this paper, we address the class of interaction effects between clustering and prediction algorithms when attempting to estimate the out-of-cluster (OOC) loss. In the self-driving vehicle example, this encompasses pixel and LIDAR point segmentation (i.e clustering tasks) and prediction tasks based on these segmentations (e.g.\ object type classification, current and future state regression). We observe this is often also a concern in domains including online shopping, medical systems and census statistics, which are further explored in the experimental section. 

To elucidate the potential behavior induced by interaction effects between clustering and prediction algorithms, consider the problem of predicting heart disease from a collection of medical records. Each patient may have several records due to multiple hospital visits but it is unlikely we are able to collect multiple records for every patient. Thus, we must find a learner which generalizes well to new patients not in our training set. The typical approach is to match records belonging to the same individual using some record linkage (i.e.\ clustering) algorithm. Then the records are split by patient into a training and validation sets, such that all records for a single patient end up in either the training or validation set. This provides an unbiased estimate of the learner's error on new patients, i.e.\ the out-of-cluster loss.



The underlying challenge in this example is that we do not have access to the oracle clustering (i.e.\ the mapping from medical records to patients), but only a noisy approximation of it from the record linkage algorithm. Even in relatively low-noise domains like medical and census, these algorithms are known to be imperfect \cite{Steorts2014,Winkler1990,Winkler2006}.  If we instead take the approach of splitting the dataset according to the \emph{approximated} patient clustering, this effectively causes samples to spill across the true training and validation folds. Some samples which should have been grouped with a validation patient may have ended up with a training patient, and vice versa, without our knowledge.  In other words, the training and validation sets are no longer conditionally independent, leading to a problem called \emph{dependency leakage} \cite{Barnes2017}.  This allows the learner to overfit to patient-specific features and optimistically biases our OOC loss estimate.  For example, if a patient's records are incorrectly clustered and samples are partitioned into both the training and validation sets, the learner is rewarded for predicting whether a patient has heart disease based on their name -- which clearly will not generalize to new patients.  This overfitting need not be so blatant.  The learner may overfit to subtle patterns in a chest x-ray, a form of bias which may be hard to identify even by experienced radiologists.

This interaction between clustering errors and a prediction algorithm is particularly dangerous because our learner may appear to be doing well on the validation set, but does far worse when we deploy it in the real world on new patients.  This is compounded by the fact that some application domains (e.g.\ medical, census) involve extreme consequences, including patient misdiagnosis and misguided public policy decisions.  Note that this bias is undetectable during standard cross-validation procedures unless an explicit attempt is made to estimate and correct for it, which is the primary focus of this paper. Saeb et al.\ note that over half of selected medical studies failed to account for any clustering, allowing records for the same patient to occur in both the training and validation datasets, a significant statistical mistake \cite{Saeb2016}.

The contributions and organization of the remainder of this paper is as follows. We begin in \cref{sec:problem} by formalizing the problem and notation. In \cref{sec:properties}, we present theoretical properties out-of-cluster prediction loss given an approximate clustering which hold under various conditions. In \cref{sec:testing}, we demonstrate how these properties can be used to construct a simple hypothesis test for the presence of bias in cross-validation results.

Computational scalability is a significant barrier to estimate bias in cross-validation results, as previous results typically scale $\mathcal{O}(n^3)$ \cite{Barnes2017}. In \cref{sec:scalable}, we systematically alleviate these concerns by proposing function approximation and matrix sketching techniques which have \emph{constant} computational complexity relative to the dataset size $n$.  Interestingly, our matrix sketching technique is able to reduce the number of columns in a key structured matrix, unlike other matrix sketching techniques which typically reduce the number of rows.  Note this does not preclude also applying standard matrix sketching techniques.

Finally, we conducted empirical studies on Parkinson's, heart disease, 1994 US Census and Dota 2 video game data, and provide results in \cref{sec:empirical} which demonstrate the practical behavior of interaction effects closely aligns with our theoretical results. Further, we deploy our scalability techniques to previously intractable problem classes, while maintaining similar error levels on smaller problems.



\section{Problem Statement} \label{sec:problem}
More formally, let $X = x_1, \dotsc, x_{n_x}$ be the $n_x$ observed samples, $y$ be the corresponding labels, and $c: \{1, \dotsc, n_x\} \to \{1, \dotsc, k\}$ be the oracle clustering algorithm which partitions the data into $k$ clusters (e.g.\ $k$ is the number of patients, $n_x$ is the number of medical records). Our high level goal is to train a prediction algorithm $f$ which generalizes to new clusters, i.e.\ has low out-of-cluster loss.  The leave-one-cluster-out (LOCO) estimator
\begin{equation} \label{eq:loco}
 \widehat{\text{Err}}_{\text{LOCO}} = \frac{1}{|c_1^{-1}|}\sum_{j \in c^{-1}_1} \ell(y_j, f(x_j \mid x_{\bar c^{-1}_1}, y_{\bar c^{-1}_1})),
\end{equation}
is an unbiased estimator of the OOC loss\footnote{An unbiased estimate of training on $k-1$ clusters.  It is slightly biased compared to training on all $k$ clusters.}. Here, $\mathcal{T} = (X_{\bar c_1^{-1}}, Y_{\bar c_1^{-1}})$ and $\mathcal{V} = (X_{c_1^{-1}}, Y_{c_1^{-1}})$ denote the training and validation sets, where $c^{-1}_i$ and $\bar c^{-1}_i$ denote all sample indices belonging and not belonging to cluster $i$, respectively.  In other words, all samples belonging to one cluster form the validation set, and samples from the remaining clusters form the training set. Without loss of generality, we have arbitrarily chosen to leave the first cluster out.

The key question here is: how will errors in the clustering algorithm $\hat c$ effect our ability to train and validate the predictor $f$? By examining the LOCO estimator used to train and validate $f$, we see that errors in $\hat c$ result in noisy training and validation sets $\mathcal{\hat T}$ and $\mathcal{\hat V}$, where some samples have flipped between $\mathcal{T}$ and $\mathcal{V}$. 

We assume that clustering errors are made probabilistically with independent rate $p_0$, an assumption similar to one used in analyzing standard supervised learning with noisy class labels \cite{Natarajan2013,Liu2016}. If the clustering algorithm provides uncertainty estimates (e.g. Bayesian methods), we believe it would be possible to incorporate this uncertainty via importance weighting.
Further, we consider the unidirectional leakage scenario where samples move from $\mathcal{V}$ to $\mathcal{T}$ to create $\mathcal{\hat V}$ and $\mathcal{\hat T}$, such that $\mathcal{\hat T} \overset{n}{\sim} M_{P_\mathcal{T}, P_\mathcal{V}}(1-p_0, p_0)$, where $M_{a, b}(w_a, w_b)$ denotes the mixture distribution of $a$ and $b$ with weights $w_a$ and $w_b$ and $p_0$ is the leakage probability (a function of $\hat c$'s error). Our results apply to the other unidirectional leakage scenario where samples move from $\mathcal{T}$ to $\mathcal{V}$, and it may be possible to extend them to the bidirectional leakage scenario using similar techniques as \cite{Barnes2017}.

If the clustering is perfect (i.e.\ $\hat c = c$), then $p_0=0$.  Let $e_i$ be the expected loss at some other $p = \sfrac{i}{n}$ fraction of corrupted samples (we use the notational shorthand $e(p)$ to denote $e_{pn}$).  The expected OOC loss is equivalent to $e_0$ (i.e.\ zero dependency leakage, $p=0$), but we only observe the empirical loss at some $p_0 > 0$. Thus, our specific goals are to characterize the behavior of the interaction effects $e$ and to efficiently estimate $e_0$ in order to train and validate $f$.

\section{Theoretical Properties} \label{sec:properties}
In this section, we present theoretical results on interaction effects between prediction and clustering algorithms. First, we prove that under mild conditions, the sequence of losses $e = e_0, e_1, \dotsc, e_{n}$ is monotonically decreasing due to dependency leakage. Second, under slightly stronger conditions, the sequence will be convex with respect to $p$.  Intuitively, errors in the clustering algorithm allows the prediction algorithm to `peak' at samples in the validation distribution, which will improve its performance with diminishing returns.  Monotonicity has previously been conjectured, but never proven, and the conditions where it holds were uncertain.  To the authors knowledge, no previous work has discussed whether $e$ is convex.
 
We say a learner $f$ is optimal under its training distribution if
\begin{equation}
f(\cdot | \mathcal{T}) \in \argmin_{f \in \mathcal{F}} \mathbb{E}_{x, y \sim P_{\mathcal{T}}} \ell(f(x), y).
\end{equation}
Generally speaking, this tends to be true for large $|\mathcal{T}|$, small model complexity of $\mathcal{F}$ or sufficient regularization in $\ell$.  This does not imply $f$ is overfit to the training set, but in fact that it generalizes well across $P_\mathcal{T}$.  

\begin{theorem} \label{thm:monotone}
The sequence $e_0, e_1, \dotsc, e_{n}$ is monotonically decreasing if $f$ is optimal under its training distribution.
\end{theorem}

\begin{proof}
See \cref{app:monotone}.
\end{proof}

\begin{remark}
This theorem implies that the interaction will always \emph{optimistically} bias our cross-validation results.  This is in fact the most dangerous type of bias, as our heart disease classifier will perform well on the off-line hold-out set, but then perform worse when we deploy it in the real world on new patients or at new hospitals.
\end{remark}

If $f$ is not optimal among $\mathcal{F}$, it is possible to construct counterexamples such that $e_0, \dotsc, e_n$ is not monotonically decreasing.  

In our second theoretical result, we show that the expected loss is convex with respect to the strength of interaction effect $p$.  Let $\ell_P(f) = \mathbb{E}_{x, y \sim P}\ell(f(x), y)$ be the expected loss of the learner $f$ under distribution $P$.  Then the following theorem holds.
\begin{theorem} \label{thm:convex}
The sequence $e_0, e_1, \dotsc, e_{n}$ is convex if $f$ is optimal under its training distribution and $\ell_{P_\mathcal{T}}$ and $\ell_{P_\mathcal{V}}$ are strictly convex and differentiable over $f$.
\end{theorem}

\begin{proof}
See \cref{app:convex}
\end{proof}

Strictly convex and differentiable loss functions hold for a wide class of problems, including support vector machines and linear or ridge regression.  

\begin{remark}
The convexity of $e$ compounds the monotonic behavior in \cref{thm:monotone}, as it implies that even a small amount of error in our clustering $\hat c$ can cause large amounts of cross-validation bias in $f$.
\end{remark}

In \cref{sec:empirical}, we empirically demonstrate both these properties hold on all examined datasets.

\section{Hypothesis Testing} \label{sec:testing}
A principal question for data scientists is whether an interaction effect exists between their clustering and prediction algorithms.  Here, we show how to use the theoretical properties from \cref{sec:properties} to quickly construct a two-sample $t$-test for dependency leakage, which avoids the complexity of constructing an estimator for the OOC loss $\hat e_0$.

Consider the alternative hypothesis $H_a: e_0 > e(p_0)$, where $p_0 > 0$ is the unknown leakage probability and $e_0$ is the OOC loss with zero leakage (i.e.  no interaction effect).  By \cref{thm:monotone}, we can use a one sided test because $e(p_0) \geq e(p_{n})$.  
First, form $n_\mathcal{T}$ training folds each of size $n^\prime$ from $\mathcal{\hat T}$.  Additionally, form $n_\mathcal{T}^\prime$ training folds of size $n^\prime$ and $n_\mathcal{T} + n_\mathcal{T}^\prime$ validation folds of size $n_\mathcal{V}$ from $\mathcal{\hat V}$.

Train and validate $f$ on the disjoint $n_\mathcal{T} + n_\mathcal{T}^\prime$ training folds and corresponding validation folds.  Let $z = z_1, \dotsc, z_{n_\mathcal{T}}$ and $z^\prime = z_1^\prime, \dotsc, z_{n_\mathcal{T}^\prime}^\prime$ be the validation loss of $f$ trained on the folds from $\mathcal{\hat T}$ and $\mathcal{\hat V}$, respectively.  Let $\bar z$ and $\bar z^\prime$ be the mean of these two sequences.  Then
\begin{equation}
\bar z - \bar z^\prime \sim N(e(p_0) - e(p_{n}), \sigma^2(\bar z) + \sigma^2(\bar z^\prime)) \nonumber
\end{equation}
and the two-sample $t$-test statistic is
\begin{equation}
T = \frac{\bar z - \bar z^\prime}{\sqrt{\frac{s_1^2}{n_\mathcal{T}} + \frac{s_2^2}{n_\mathcal{T}^\prime }}}
\end{equation}
where $s_1^2$ and $s_2^2$ are the sample variances of $z$ and $z^\prime$, respectively.

Rejecting the null hypothesis $H_0^\prime: e(p_0) \leq e(p_{n})$ when $T > t_{1-\alpha, v}$ is a level $\alpha$ test, where $t_{1-\alpha, v}$ is the critical value of the $t$-distribution with $v$ degrees of freedom.  Further, by \cref{thm:monotone} and \cref{thm:convex}, $e(p_0) \neq e(p_{n}) \Rightarrow e_0 \neq e(p_0)$ so long as $p_0 > 0$.  Thus, rejecting the null hypothesis $H_0: e_0 \neq e(p_0)$ when $T > t_{1-\alpha, v}$ is also a level $\alpha$ test.

There are two takeaways to consider when using this test. The first powerful property is that it does not require actually knowing the clustering error or leakage probability  $p_0$ a priori, only that it is not perfect (a very weak assumption). Second, the Type II error rate of this test largely depends on the convexity of $e$.  If $p_0 < 0.5$ and $e$ is linear, then $e(p_0) - e(p_n) > e_0 - e(p_0)$ and the Type II error rate will actually be \emph{lower} than if we could directly test $e_0 \neq e(p_0)$.  Conversely, the Type II error rate becomes larger as $e$ becomes more strongly convex.

\section{Scalable Estimators} \label{sec:scalable}

\begin{algorithm}[t]
\caption{B3: Binomial Block Bootstrap}\label{alg:b3}
\begin{algorithmic}[1]
\Procedure{B3}{$f, \mathcal{\hat T}, \mathcal{\hat V}, p, n^\prime, t$}
\State $\bar b \gets \vec 0$
\For{$p_i$ in $p$}
\State $p^\prime \gets \frac{p_i - p_0}{1-p_0}$
\For{$j\gets 1$ to $t$}
\State $\mathcal{T}_j^\prime \overset{n^\prime}{\sim} M_{\mathcal{\hat T}, \mathcal{\hat B}}(1-p^\prime, p^\prime)$  
\State $\mathcal{V}_j^\prime \gets \mathcal{\hat V} \setminus \mathcal{T}^\prime_j$

\State $\hat b_i \gets \frac{1}{|\mathcal{V}_j^\prime|} \sum_{(x, y) \in \mathcal{V}_j^\prime} \ell (y, f(x\mid\mathcal{T}_j^\prime))$ 
\State $\bar b_i \gets \bar b_i + \frac{\hat b_i}{t}$
\EndFor
\EndFor
\State $A_{ij} \gets \mathbb{P}(\textup{Binomial}(n^\prime, p_i) = j)$
\State $\hat e, residual \gets A(A^\intercal A)^{-1}A^\intercal \bar b$
\State \Return $\hat e_0, residual$
\EndProcedure
\end{algorithmic}
\end{algorithm}

Existing asymptotically unbiased estimators for the OOC loss are limited by their need to solve a linear system of $n$ variables, where $n$ is the size of the bootstrap training set \cite{Barnes2017}.  In this section, we present two approaches for dramatically improving the computational efficiency of the unidirectional, known $p_0$ variant of the Binomial Block Bootstrap (B3) estimator, which is the core method of other variants.

We begin by recapping the Binomial Block Bootstrap (B3) estimator for the OOC loss, shown in \cref{alg:b3} \cite{Barnes2017}. The method leverages the fact that a resample with replacement from $\mathcal{\hat T}$ can be written as a binomial expectation over $e$ (see row 1 of \cref{eq:intuition}), by definition of the binomial distribution. Second, by adding additional corruption into $\mathcal{\hat T}$ in the form of samples from $\mathcal{\hat V}$, it effectively increases the clustering error $p_0$ to $p_1$ (see row 2 of \cref{eq:intuition}). Repeated operation of these principles allows constructing the fully defined linear system
\begin{align}
\begin{blockarray}{ccc}
& & 0 \quad 1 \quad \cdot \quad \cdot \quad n \\
\begin{block}{cc(c)}
  p_0 & & \leftarrow  \textup{Binomial pmf}  \rightarrow \\
  p_1 &  &   \cdot   \\
  \cdot &  &   \cdot   \\
  \cdot &  &   \cdot   \\
  1 & & \leftarrow  \textup{Binomial pmf}  \rightarrow \\
\end{block}
\end{blockarray}
\quad \begin{blockarray}{c}
  \\
\begin{block}{(c)}
   \\
  \\
  e   \\
   \\
   \\
\end{block}
\end{blockarray}
\; &=\;
\begin{blockarray}{c}
 \\
\begin{block}{(c)}
   \\
  \\
  b   \\
   \\
   \\
\end{block}
\end{blockarray} \label{eq:intuition} \\
A(p_0) \hspace{5.5em} e\quad  &= \quad b \nonumber
\end{align}
where matrix $A \in \mathbb{R}^{m \times (n+1)}$ is defined by
\begin{equation}
A_{ij} = \mathbb{P}(\textup{Binomial}(n, p_i) = j) \label{eq:A}.
\end{equation}
Vector $b$ is formed by the average empirical loss of repeatedly sampling with replacement from $\mathcal{\hat T}$ and $\mathcal{\hat V}$ at $m$ increasing levels of corruption $p_0, p_1, \dotsc, 1$.  The insight of this method is that by \emph{increasing} dependency leakage by further mixing $\mathcal{\hat T}$ and $\mathcal{\hat V}$ and thus increasing $p$ from $p_0$ towards $1$, one can extrapolate the loss at zero clustering error, using the structured matrix $A$.  Although the B3 estimator is asymptotically unbiased, solving the linear system has $\mathcal{O}(n^{3})$ cost, and forming the loss estimate $b$ has $\mathcal{O}(n)$ computational cost.  If the prediction algorithm $f$ has an expensive training procedure (e.g.\ deep neural networks), the latter term may outweigh the former due to a large fixed constant. 

\subsection{Basis Function Approximation}
Perhaps the most straightforward approach to scaling these estimators is through function approximation, which also conveniently provides a natural form of regularization.  We parameterize $e$ by a set of $s$ basis functions $\psi_1, \dotsc, \psi_s$, such that
\begin{equation} \label{eq:scale-basis}
e_i = \xi_1 \psi_1(i) + \xi_{2} \psi_{2}(i) + \dotsc + \xi_s \psi_s(i)
\end{equation}
where $\xi_{1}, \dotsc, \xi_{s} = \xi \in \mathbb{R}^s$ are the $s$ parameters.  Then $e = \Psi \xi$ where $\Psi \in \mathbb{R}^{(n + 1) \times s}$ is the matrix of basis values.

Instead of solving the linear system $Ae = b$, where $A \in \mathbb{R}^{m \times (n + 1)}$ and we choose $m \geq n$, we can now solve $A^\prime \Psi \xi = b^\prime$ where $A^\prime \in \mathbb{R}^{m^\prime \times (s + 1)}$ and we choose $m^\prime \geq s$.  Note the size of this system no longer depends on the number of samples $n$.  Instead, it depends on the number of parameters in our approximation of $e$, which will be a fixed constant.  This new linear system is well behaved, depending on the choice of basis function $\psi$.

\begin{theorem}
Let $\psi_0, \dotsc, \psi_s$ be a set of $s$ unisolvent, bounded and continuous functions over $[0, 1]$ and let
\begin{equation*}
e_i = \xi_0\psi_0\left(\frac{i}{n}\right) + \dotsc + \xi_s\psi_s\left(\frac{i}{n}\right).
\end{equation*}
Then $A\Psi$ is invertible as $n \to \infty$.
\label{thm:basis}
\end{theorem}
\begin{proof}
See \cref{app:basis_proof}.
\end{proof}

\subsection{Matrix Sketching}
Second, we propose a new matrix sketching technique which reduces the number of columns in the structured matrix $A$.  Unlike typical matrix sketching techniques, which reduce the number of rows, we are able to reduce the number of columns and thus the dimensionality of the solution $e$ by leveraging the structure in $A$ and properties of $e$ from \cref{thm:monotone}.  After reducing the number of columns, one could further apply standard matrix sketching techniques to also reduce the number of rows.  Our algorithm guarantees recovering $e_0$ within a linear factor of the true value.

Consider the setting where $m \leq n$ and the system $Ae = b$ is underdetermined.  This is especially relevant for large datasets, where it is computationally infeasible to sample at $m > n$ levels of leakage or perhaps even solve for $n$ unknowns.  Let $S \in \mathbb{R}^{m \times (k+1)}$, $m > k$ be our sketching matrix, where $S$ is formed such that the first column of $S$ equals the first column of $A$, i.e.\ $S_0 = A_0$.  Partition the remaining $n$ columns of $A$ into $k$ sets, for example using $k$-medoids or simply grouping adjacent columns together (since by the definition of $A$, these will be close together).  Let $r: \{0, \dotsc, n\} \to \{0, \dotsc, k\}$ be the resulting partition, where $r(0) = 0$ is the singleton partition of the first column.  Finally, form the remaining columns of $S$ from the medoids of the $k+1$ sets.  Each column in $A$ is within an $\epsilon$-ball of at least one column in $S$, i.e.
\begin{equation}
\epsilon = \underset{i \in \{0, \dotsc, n\}}{\textup{max}}\Vert A_i - S_{r(i)}\Vert  \nonumber
\end{equation}

\begin{theorem} \label{thm:sketch}
Let $e^\prime$ be the solution to the sketched system $Se^\prime = b$ and $s$ be the first row of $s^{-1}$.  The error between the true and sketched solution is bounded by
\begin{equation}
|e_0^\prime - e_0| \leq \epsilon n \Vert s^\prime\Vert e_0.
\end{equation}
\end{theorem}

\begin{proof}
See \cref{app:sketch_proof}
\end{proof}

\subsection{Connection to B\'ezier curves and Bernstein polynomials}
The B3 estimator in \cref{eq:intuition} has close ties to the Bernstein basis and B\'ezier curves which have not previously been realized.  Notice that each column of $A$ corresponds to a Bernstein basis function evaluated at at $p_0, \dotsc, 1$.  Thus, the B3 estimator is equivalent to solving for the Bernstein coefficients or B\'ezier control points $e$, where the system is constructed through the B3's bootstrapping process. A more detailed mathematical connection is provided in \cref{app:bernstein}.

\section{Empirical study} \label{sec:empirical}

Finally, we conducted an empirical, finite-sample study which validates the theoretical properties in \cref{sec:properties} and demonstrates the computational speed-ups provided in \cref{sec:scalable}. For comparison, we consider three benchmark estimators for the OOC loss -- IID, LOCO and the B3 estimator with a fourth order trend filter and monotonic regularizers (T4+mono). The latter is the empirical state-of-the-art method, though suffers from computational scalability issues. IID is the typical cross-validation split, where samples are uniformly randomly split into training and validation sets, which does not account for the latent clustering. LOCO is the leave-one-cluster-out estimator described in \cref{eq:loco} using an approximated clustering $\hat c$ with an error of $p_0 = 0.1$.

In all experiments, we used a linear SVM as the predictor $f$. This is a best-case scenario, as interaction effects depend on the predictor $f$'s ability to overfit to mistakes from the clustering algorithm $\hat c$. Thus, as the complexity of the predictor class increases, the interaction effect worsens.

\begin{figure}[t] \label{fig:times}
    \centering
    \includegraphics[width=\columnwidth]{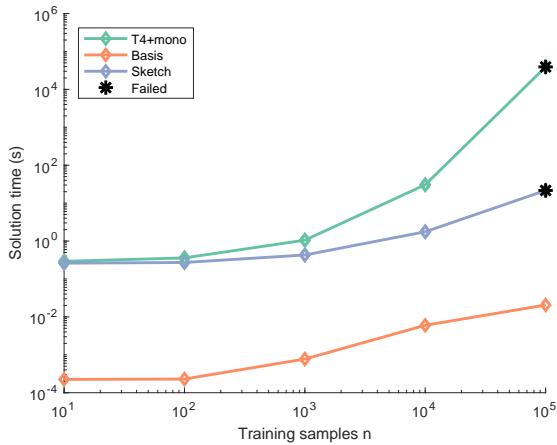}
    \caption{Computational scalability results on synthetically generated datasets. Our methods (Sketch, and in particular, Basis) are significantly faster than existing methods (T4+mono). ``Failed'' indicates the SDPT3 solver failed to find an accurate solution.}
\end{figure}

\begin{figure*}[t]
    \begin{subfigure}[t]{0.495\textwidth}
        \centering
        \includegraphics[width=\textwidth]{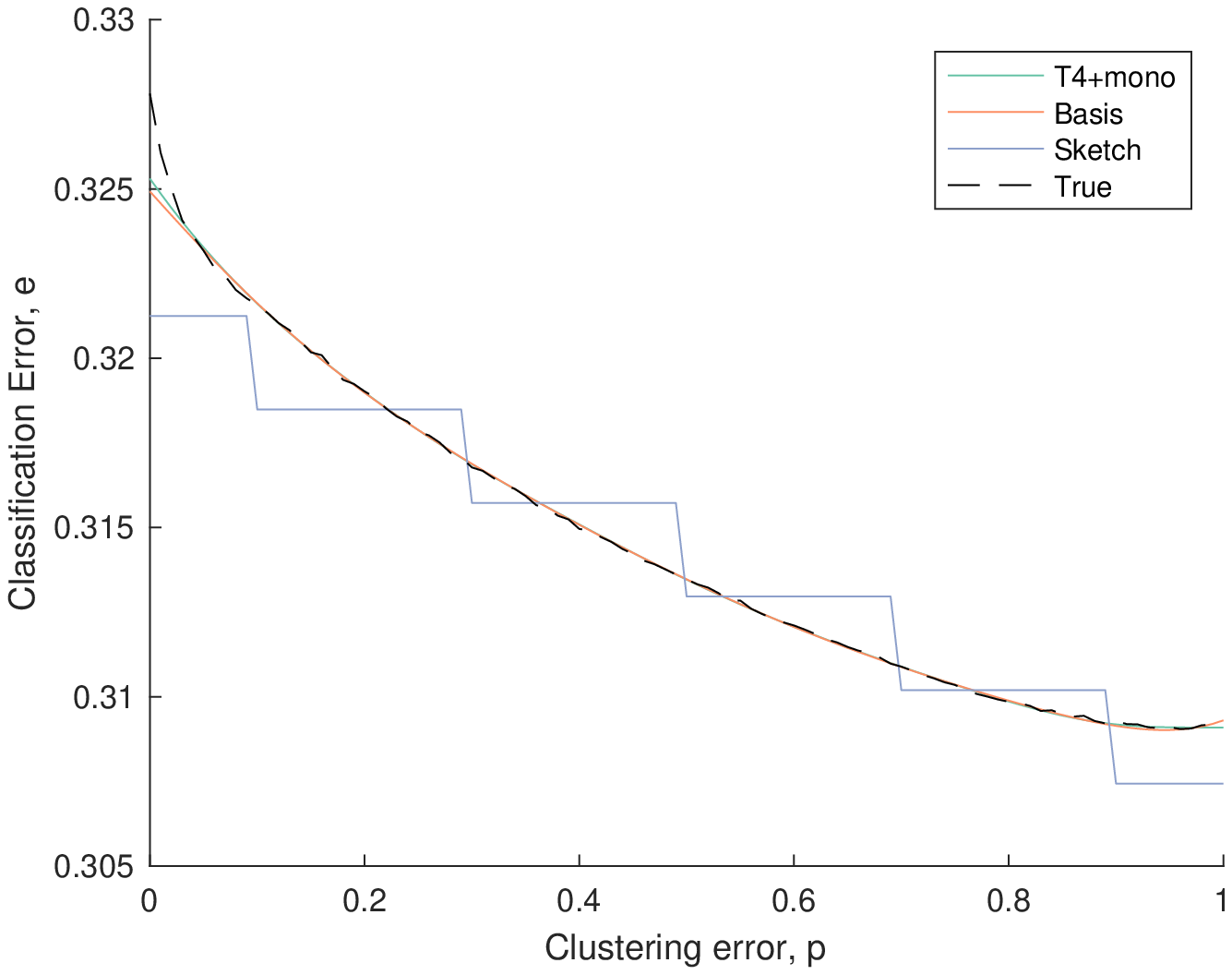}
        \caption{1994 US Census}
    \end{subfigure}
    \begin{subfigure}[t]{0.495\textwidth}
        \centering
        \includegraphics[width=\textwidth]{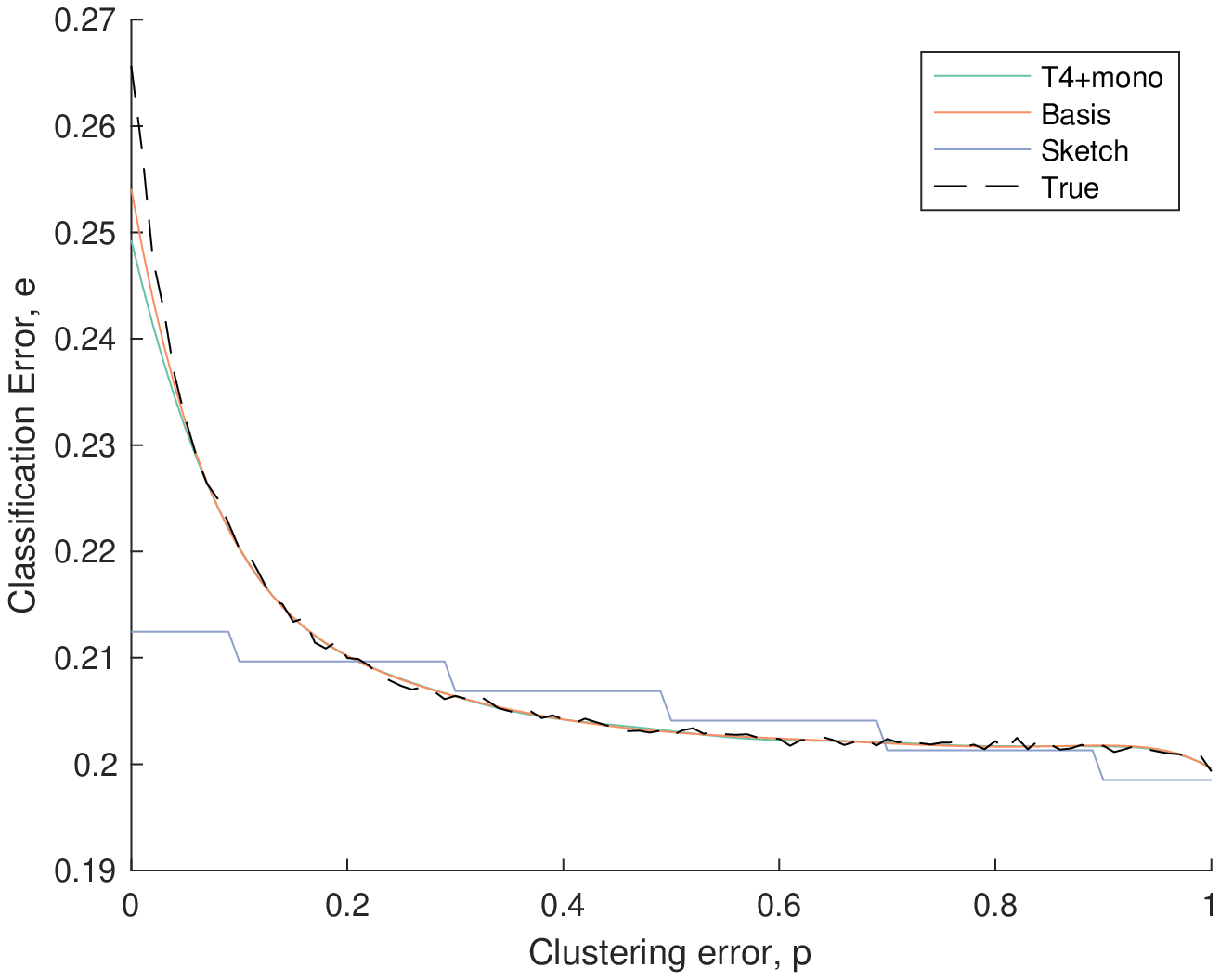}
        \caption{Heart Disease}
    \end{subfigure}
    \caption{Empirical results show the loss is indeed convex and monotonically decreasing, validating our theoretical results in \cref{sec:properties}.  Note our methods are able to recover the full loss in addition to the true OOC loss $e_0$. Plots for the remaining experiments are provided in \cref{app:additional}}.
    \label{fig:fullfit}
\end{figure*}

\begin{figure*}[t]
    \begin{subfigure}[t]{0.495\textwidth}
        \centering
        \includegraphics[width=\textwidth]{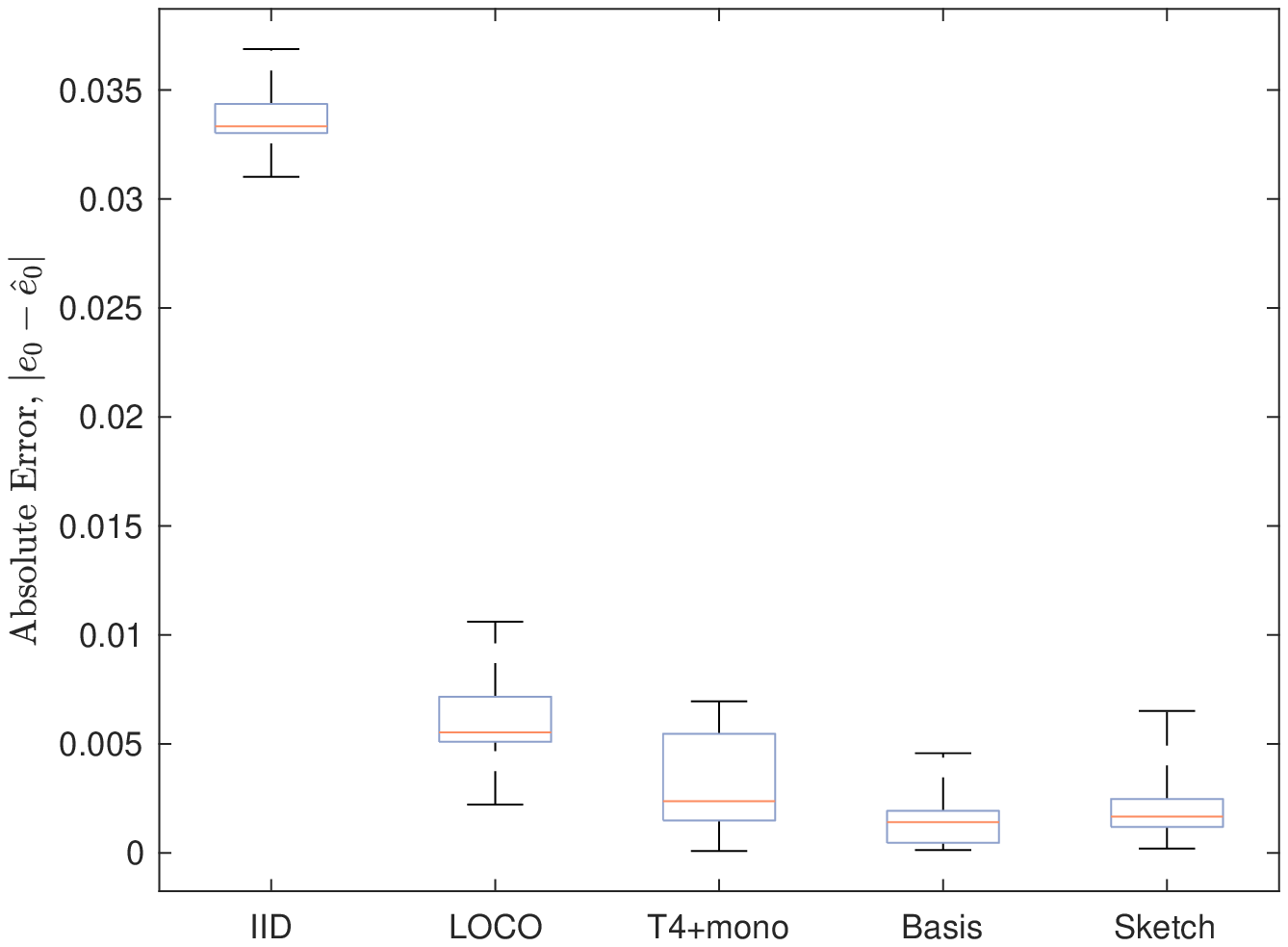}
        \caption{Parkinson's}
    \end{subfigure}
    \begin{subfigure}[t]{0.495\textwidth}
        \centering
        \includegraphics[width=\textwidth]{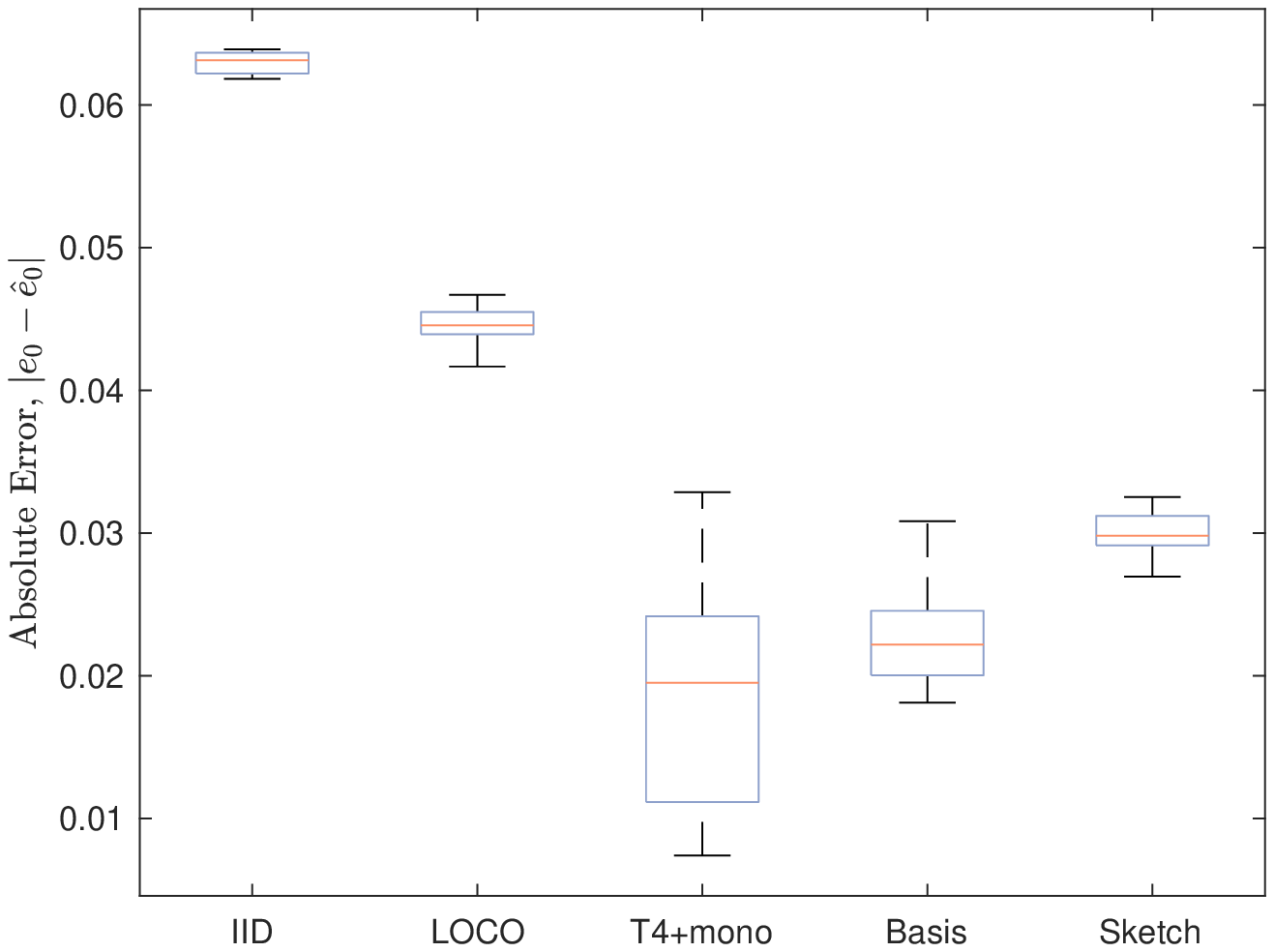}
        \caption{Heart Disease}
    \end{subfigure}
    \\
    \begin{subfigure}[t]{0.495\textwidth}
        \centering
        \includegraphics[width=\textwidth]{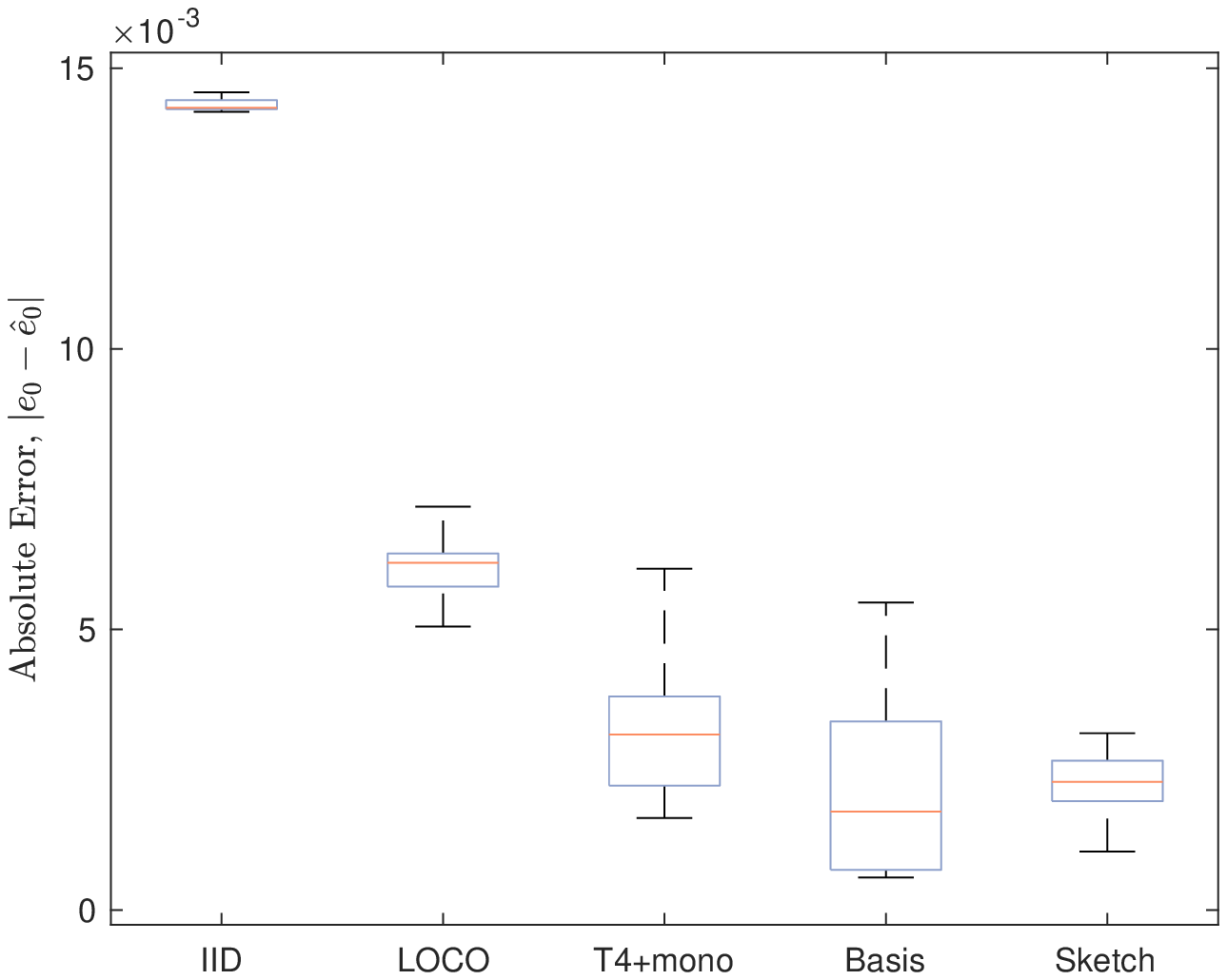}
        \caption{1994 US Census}
        \label{fig:adult}
    \end{subfigure}
    \begin{subfigure}[t]{0.495\textwidth}
        \centering
        \includegraphics[width=\textwidth]{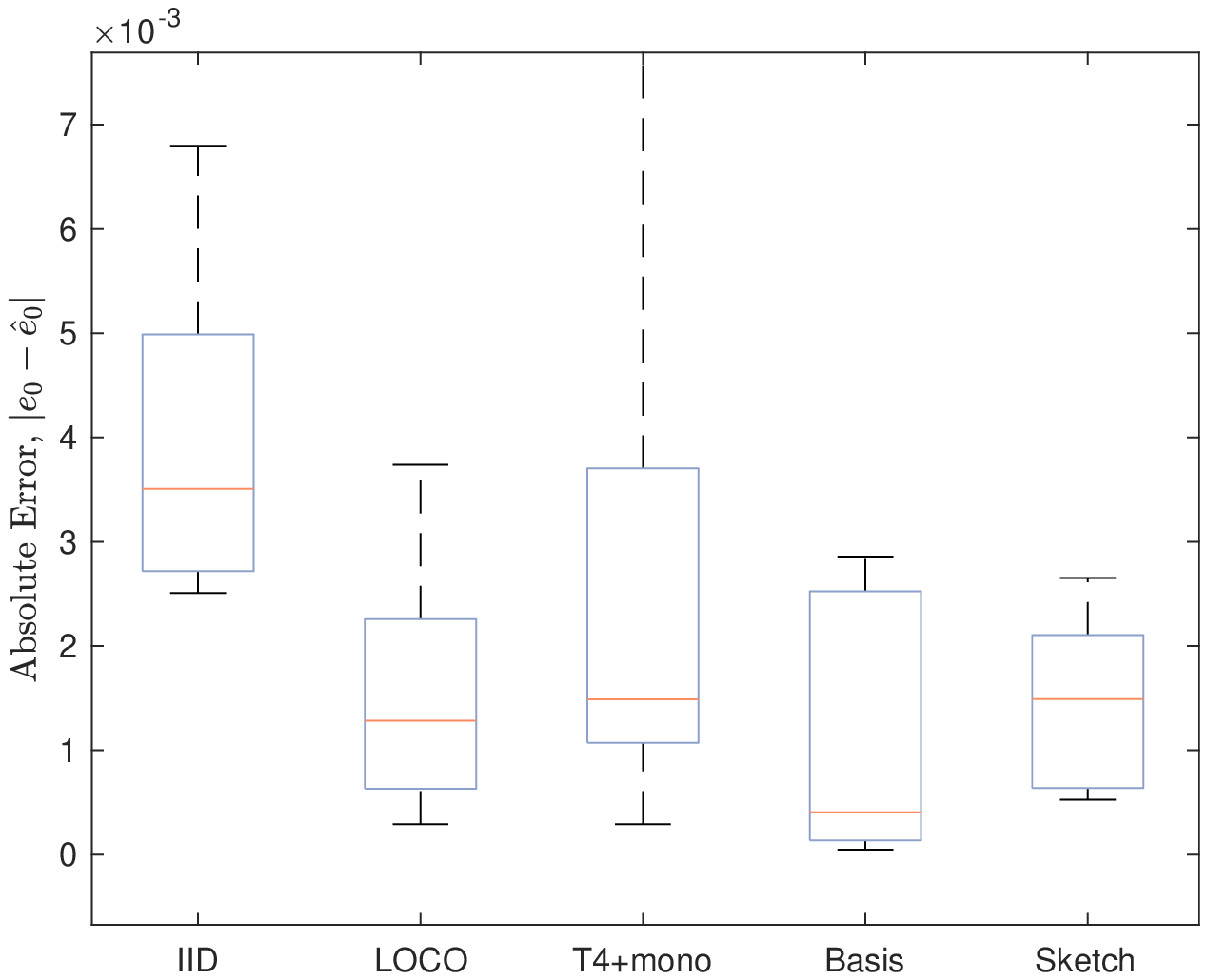}
        \caption{Dota 2}
    \end{subfigure}
    \caption{Estimating the OOC loss $e_0$.  Our function approximation and novel matrix sketching techniques perform comparably to existing methods at significantly reduced computational cost.}
    \label{fig:boxplots}
\end{figure*}

Note that in order to compute the error of our estimators, we are required to use a dataset where the oracle clustering is indeed available. For many of these experiments, we used data collected in very controlled settings to guarantee no clustering error in the ground truth. In more practical scenarios, this information would not be available. We formed the training and validation sets by splitting the approximate clusters according to \cref{sec:problem}, and we controlled clustering errors by flipping samples from $\mathcal{V}$ to $\mathcal{T}$ with uniform, i.i.d. rate according to $p_0$. Complete experimental details are provided in \cref{app:details}.

\begin{table}
  \caption{Computational timing results demonstrate our methods, and in particular the basis function approximation technique, are significantly faster than the previous state-of-the-art B3 estimator with fourth order trend filter and monotonicity constraint (T4+mono). Results shown in seconds.}
  \centering
  \begin{tabular}{lccc}
    \toprule
    & \multicolumn{3}{c}{Method}                   \\
    \cmidrule(r){2-4}
    Dataset     & T4+mono     & Sketching & Basis \\
    \midrule
    1994 US Census     & 0.5662       & 0.4059 & \bf{7.822e-5}   \\
    Heart     & 0.5847 & 0.4105 & \bf{6.582e-5}      \\
    Parkinson's & 0.6194 & 0.4338 & \bf{2.043e-5} \\
    Dota 2    & 1.0965 & 0.4678 & \bf{1.946e-5} \\
    \bottomrule
  \end{tabular}
  \label{tab:times}
\end{table}

\paragraph{Computational Scalability}
The proposed approximation techniques, and especially the basis function approximation technique, are faster than existing OOC estimators and are tractable on larger problem classes. To compare performance across a large range of dataset sizes, we generated increasingly large synthetic training sets and compared solution times in \cref{fig:times}. All methods used only 10 corruption levels (i.e.\ the number of rows in $A$), the smallest reasonable number required to find an accurate solution. We observed that increasing the number of rows in $A$ exponentially increased solution times. Thus, these results are likely the largest datasets appropriate for existing methods. In particular, notice that the solver failed to find accurate solutions on the largest problem class for all methods except for with the basis approximation technique.

Timing results on real world datasets (described in the following sections) are reported in \cref{tab:times}. Similarly, we find the basis approximation technique is the fastest by several orders of magnitude.

Constrained linear programs (e.g.  T4+mono, sketching) were solved using SDPT3's infeasible path-following algorithm, for unconstrained linear systems we took advantage of fast QR solvers (a major reason the basis method is so efficient). All optimizations were performed using an Xeon Gold 6152 CPU @ 2.10GHz and 754 GB RAM. We found that T4+mono, and to a lesser extent, the sketching approximation, required the majority of this memory for the largest problem classes.

\paragraph{Interaction Characteristics}
\cref{fig:fullfit} demonstrates that interaction effects between the clustering and prediction algorithm cause the cross-validation error $e$ to decay monotonically and convexly, as predicted by \cref{thm:monotone} and \cref{thm:convex}. This visually demonstrates the expected adverse behavior -- if our clustering algorithm makes even a few mistakes, we may think our predictor has a low error rate, but when we deploy it in the real world on new clusters, it will perform far worse. Empirically, the interaction biases cross-validation results by upwards of 25\%, but our methods are largely able to correct for this bias. Note our methods not only recover the true OOC loss $e_0$, but also the entire curve $e_1, e_2, \dotsc, e_n$.

\paragraph{Estimator Error}
Finally, we empirically test whether the approximations introduce any additional error into the OOC estimate. \cref{fig:boxplots} shows that the approximations perform comparably to the previous state-of-the-art T4+mono estimator, at significantly reduced computational cost. The specific experiments are briefly described below, see \cref{app:details} for details.

\emph{Parkinson's} \quad
In the first experiment, we attempted to predict whether a patient has Parkinson's disease based on multiple voice recordings featurized according to doctor specifications \cite{UCI}. Here, each cluster corresponds to an individual, and each cluster contains multiple voice recordings. The OOC error corresponds to the ability to predict Parkinson's on new individuals not in the training set.

\emph{Heart Disease} \quad
In the second experiment, we use medical records from four hospitals in Switzerland; Hungary; Cleveland and Long Beach, USA \cite{UCI}.  Given a patient medical record, including vital signs and demographics, the task is to train a heart disease classifier which performs well at new, previously unseen hospitals.  Since we do not have access to multiple records per patient, we instead choose the related task of generalizing across hospital clusters.  

\emph{1994 US Census} \quad
In the third experiment, we consider the issue of machine learning bias against certain populations in the 1994 US Census due to interaction effects \cite{UCI}.  Given a person's occupational, educational and demographic information, our task is to predict whether a person's income is greater than US\$50k per year (finer resolution income data cannot be publicly disclosed).  In particular, we wish to train a classifier which performs well across people from all origin countries.  We arbitrarily chose Indian and Canadian immigrants as our leave-one-out clusters, and natural born citizens, Salvadoran, German, Mexican, Filipino and Puerto Rican immigrants as our training set\footnote{We left out two clusters due to a small number of immigrants from some countries in the dataset.}.  

Our results, presented in \cref{fig:adult}, validates our claim that interaction effects can bias our learner against certain populations.  The SVM classifier learns attributes specific to the corrupted samples which spilled from the validation set into the training set, even though they do not generalize to all immigrants.

\emph{Dota 2} \quad
In the final experiment, we attempt to predict the winner of a Dota 2 video game based on the heroes each team selects at the beginning of the game. This is equivalent to learning an undiscounted value function for a binary, sparse reward function in reinforcement learning. Here, clusters correspond to the type of game played, and we wish to learn a predictor $f$ which generalizes across new game types.

\section{Previous Work}
Previous work has studied various aspects of learning with dependent data, beginning with the necessity of independence for the naive bootstrap \cite{Singh1981}.  Subsequent work has proceeded along two directions: most prominently for time-series data, but also for cluster data.  In time-series data, a stochastic process defines the data dependency, which usually decreases over larger time intervals \cite{Hall1985, Liu1992}.  The common approach to limiting dependency and thus controlling estimator bias and variance is to form blocks of data which are sufficiently spaced in the time domain \cite{Lahiri2003}.

In the clustering setting, bootstrap methods have been proposed for a variety of problem formulations, roughly categorized into model-based and  model-free methods \cite{Cameron2015}.  The first, model-based line of work directly models the within-cluster error correlation, a relatively strong data assumption.  The second, model-free line of work 
performs post-estimation bias-correction for least squares \cite{White1984}, small or unbalanced number of clusters \cite{MacKinnon2017} and non-linear settings \cite{Liang1986}.  Other authors have shown asymptotic analysis results for the residual bootstrap \cite{Andersson2001}, randomized cluster bootstrap, two-stage bootstrap \cite{Davison1997} and multi-way bootstrap \cite{Field2007,Miller2011}.  The fundamental difference in our work is we do not assume samples in different clusters $\hat c$ are independent, due to mistakes in the clustering algorithm.

In \cref{sec:scalable}, we built off the work of \cite{Barnes2017}, who introduced the first asymptotically exact estimator of the OOC loss with clustering errors.  However, they failed to characterize the behavior of the interaction effect, and their estimator scaled $\mathcal{O}(n^3)$, which limited its applicability to smaller problems.

\section{Conclusion}
We argued that interaction effects between clustering and prediction algorithms can cause dangerous and elusive behavior when estimating the out-of-cluster loss in machine learning systems. 
We theoretically characterized when and how this interaction behavior is exhibited, and demonstrated these properties hold in practice on all examined datasets. In particular, we showed the out-of-cluster loss bias is convex and monotonically decreasing -- implying that even a small clustering error can significantly and optimistically bias cross-validation results.
Further, these theoretical properties are necessary to construct the statistical hypothesis test in \cref{sec:testing}, an important practical takeaway to detect for OOC bias.
Our newly introduced estimators are able to correct for this bias at significantly reduced computational cost compared to existing estimators, making the proposed approaches scalable to a wide range of practical applications.

The interaction between clustering and prediction algorithms is one common instance of an interaction effect. \cite{Sculley2014} discussed several other issues in complex machine learning systems, including hidden feedback loops and undeclared data dependencies, which may warrant further exploration.


\subsubsection*{Acknowledgements}
We thank Kin Gutierrez Olivares for the insightful connection to B\'ezier curves. This work was partially supported by DARPA (FA8750-14-2-0244 and FA8750-17-2-0130) and AFRL (FA8750-17-2-0212). Matt was partially supported by an NSF graduate research fellowship.

\clearpage
\bibliographystyle{unsrt}
{
\bibliography{aistats_2019}}

\clearpage
\appendix
\onecolumn
\begin{center}
    \LARGE{Appendix for ``On the Interaction \\ Effects Between Prediction and Clustering''}
\end{center}

\section{Proofs} \label{app:proofs}
\subsection{\cref{thm:monotone}} \label{app:monotone}
We begin by introducing a key lemma about the minimization of function mixtures.
\begin{lemma} \label{lem:monotone}
For functions $a, b: \Theta \to \mathbb{R}$ and $\alpha \in [0, 1]$, 
\begin{gather*}
a(\argmin_{\theta \in \Theta}(\alpha a(\theta) + (1-\alpha) b(\theta))) \\
b(\argmin_{\theta \in \Theta}(\alpha a(\theta) + (1-\alpha) b(\theta)))
\end{gather*}
are monotonically decreasing and increasing, respectively, with respect to $\alpha$.
\end{lemma}
\begin{proof}
Let $\Delta(\theta) = a(\theta) - b(\theta)$, $1 \geq j > i \geq 0$ and
\begin{gather}
\theta_i \in \argmin_{\theta \in \Theta} b(\theta) + i\Delta(\theta) \nonumber \\
\theta_j \in \argmin_{\theta \in \Theta} b(\theta) + j\Delta(\theta) \nonumber
\end{gather}
Then $a$ is monotonically decreasing with respect to $\alpha$  if and only if $a(\theta_i) \geq a(\theta_j)$.

\textbf{Case 1:} $\theta_i = \theta_j$.  Then $a(\theta_i) = a(\theta_j)$, $b(\theta_i) = b(\theta_j)$ and the statements holds.

\textbf{Case 2:} $\theta_i \neq \theta_j$.  Then both the following conditions must be true.
\begin{gather}
b(\theta_j) - b(\theta_i) + i\Delta(\theta_j) - i\Delta(\theta_i) > 0 \label{eq:opti} \\
b(\theta_j) - b(\theta_i) + j\Delta(\theta_j) - j\Delta(\theta_i) < 0 \label{eq:optj}
\end{gather}
If \cref{eq:opti} did not hold, then $\theta_j$ would have been optimal at $\alpha = i$, i.e.\ $\theta_j \in \argmin_{\theta \in \Theta} b(\theta) + i\Delta(\theta)$.  Likewise, if \cref{eq:optj} did not hold, then $\theta_i \in \argmin_{\theta \in \Theta} b(\theta) + j\Delta(\theta)$.

Together, they imply
\begin{align}
b(\theta_j) - b(\theta_i) + i\Delta(\theta_j) - i\Delta(\theta_i) &> b(\theta_j) - b(\theta_i) + j\Delta(\theta_j) - j\Delta(\theta_i) \nonumber \\
i\Delta(\theta_j) - i\Delta(\theta_i) &> j\Delta(\theta_j) - j\Delta(\theta_i) \nonumber \\
(i - j)(\Delta(\theta_j) - \Delta(\theta_i)) &> 0 \nonumber \\
\Delta(\theta_j) - \Delta(\theta_i) &< 0 \nonumber
\end{align}
since $i - j < 0$.  Plugging this into \cref{eq:opti},
\begin{align}
b(\theta_j) - b(\theta_i) + i\Delta(\theta_j) - i\Delta(\theta_i) &> 0 \nonumber \\
b(\theta_j) - b(\theta_i) &> i(\Delta(\theta_i) - \Delta(\theta_j)) \nonumber \\
b(\theta_j) - b(\theta_i) &> 0 \label{eq:bineq}
\end{align}
which proves the second statement.  Finally, plugging \cref{eq:bineq} into \cref{eq:optj} concludes the proof.
\begin{align}
(1-j)(b(\theta_j) - b(\theta_i)) + j(a(\theta_j) - a(\theta_i)) &<0 \nonumber \\
a(\theta_j) - a(\theta_i) &< 0 \nonumber
\end{align}
\end{proof}

The proof of \cref{thm:monotone} follows.
\begin{proof}
\textbf{Direction $\mathcal{V}$ to $\mathcal{T}$}
We say $f$ is optimal under its training distribution if
\begin{equation}
f(\cdot | \mathcal{T}) \in \argmin_{f \in \mathcal{F}} \mathbb{E}_{x, y \sim P_{\mathcal{T}}} \ell(f(x), y).  \nonumber
\end{equation}
Let $f_0, f_1, \dotsc, f_{n}$ be models learned at each level of dependency leakage, such that each model is optimal under its training distribution, i.e.
\begin{equation}
f_i \in \argmin_{f \in \mathcal{F}} \mathbb{E}_{x, y \sim M_{P_\mathcal{T}, P_\mathcal{V}}(1-\frac{i}{n}, \frac{i}{n})} \ell(f(x), y).  \nonumber
\end{equation}
The sequence $e_0, e_1, \dotsc, e_{n}$ is monotonically decreasing when
\begin{equation}
e_i - e_{i+1} \geq 0 \quad \forall i \in \{0, \dotsc, n-1\}.  \nonumber
\end{equation}
Starting from the definition of $e$ and using the notational shorthand $\ell_P(f) = \mathbb{E}_{x, y \sim P}\ell(f(x), y)$,
\begin{align}
e_i &= \mathbb{E}_{x, y \sim P_\mathcal{V}} \ell(f_i(x), y) \nonumber \\
 &= \ell_{P_\mathcal{V}}(f_i) \nonumber \\
&= \ell_{P_\mathcal{V}}(\argmin_{f \in \mathcal{F}} \mathbb{E}_{x, y \sim M_{P_\mathcal{T}, P_\mathcal{V}}\left(1-\frac{i}{n}, \frac{i}{n}\right)} \ell(f(x), y)) \nonumber \\
&= \ell_{P_\mathcal{V}}\left(\argmin_{f \in \mathcal{F}} \frac{i}{n} \ell_{P_\mathcal{V}}(f) + \left(1 - \frac{i}{n}\right) \ell_{P_\mathcal{T}}(f) \right)  \nonumber \\
\end{align}
By \cref{lem:monotone}, $e$ is monotonically decreasing with respect to $\frac{i}{n}$, and thus also with respect to $i$ since $n$ is a fixed constant.

\textbf{Direction $\mathcal{T}$ to $\mathcal{V}$.} In this direction, $e$ will further be linear:
\begin{align*}
e_0 &= \mathbb{E}_{x, y \sim P_\mathcal{V}, \mathcal{T} \overset{n}{\sim} P_\mathcal{T}} \ell(f(x|\mathcal{T}), y) \\
e_n &= \mathbb{E}_{x, y \sim P_\mathcal{T}, \mathcal{T} \overset{n}{\sim} P_\mathcal{T}} \ell(f(x|\mathcal{T}), y) \\
e_i &= \mathbb{E}_{x, y \sim M_{P_\mathcal{T}, P_\mathcal{V}}\left(\frac{i}{n}, 1 - \frac{i}{n}\right), \mathcal{T} \overset{n}{\sim} P_\mathcal{T}} \ell(f(x|\mathcal{T}), y) \\
&= \left(\frac{i}{n}\right)\mathbb{E}_{x, y \sim P_\mathcal{T}, \mathcal{T} \overset{n}{\sim} P_\mathcal{T}} \ell(f(x|\mathcal{T}), y) + \left(1 - \frac{i}{n}\right)\mathbb{E}_{x, y \sim P_\mathcal{V}, \mathcal{T} \overset{n}{\sim} P_\mathcal{T}} \ell(f(x|\mathcal{T}), y) \\
&= \left(\frac{i}{n}\right) e_n + \left(1 - \frac{i}{n}\right)e_0
\end{align*}
and $e_n \leq e_0$ by the assumption that $f$ is optimal under its training distribution.
\end{proof}

\subsection{\cref{thm:convex}} \label{app:convex}
We begin by introducing a lemma on the minimization of mixtures of convex functions.
\begin{lemma} \label{lem:convex}
For $\alpha \in [0,1]$, let $a, b: \Theta \to \mathbb{R}$ be strictly convex and differentiable (where $\dot a$ denotes $\frac{\partial a}{\partial \theta}$) over
\begin{align*}
\Theta^* &= \{ \theta \in \argmin_{\theta \in \Theta}(\alpha b(\theta) + \alpha \Delta(\theta)) \} \quad \forall \alpha \in [0, 1] \\
&= \{ \theta \in g(\alpha) \} \quad \forall \alpha \in [0, 1] \subseteq \Theta.
\end{align*}
If $\frac{\dot a}{\dot b}$ is convex, decreasing over $\Theta^*$, then
\begin{equation*}
a(\argmin_{\theta \in \Theta}(\alpha a(\theta) + (1 - \alpha)b(\theta)))
\end{equation*}
is convex over $\alpha$.
\end{lemma}

\begin{proof}
If $\frac{\dot a}{\dot b}$ is convex, decreasing then $\frac{-\dot b}{\dot \Delta}$ is also convex decreasing.
\begin{equation}
\frac{\dot a}{\dot b} \textup{convex, decreasing} \Leftrightarrow \frac{-\dot \Delta}{\dot b} \textup{concave, increasing}
\end{equation}
because $\frac{-\dot \Delta}{\dot b} = \frac{\dot b - \dot a}{\dot b} = 1 - \frac{\dot a}{\dot b}$.

Further, we know $\frac{-\dot \Delta}{\dot b} \geq 0$ because $\dot a \leq 0$ and $\dot b \geq 0$ by \cref{lem:monotone}.  Then $\frac{-\dot b}{\dot \Delta}$ is convex decreasing by the composition of the convex, decreasing function $\frac{1}{x}$ and the concave increasing $\frac{-\dot \Delta}{\dot b}$.  Note in the case where $\dot \Delta = 0$, $g(\alpha)$ is constant and the lemma holds.

At the minimum of $b(\theta) + \alpha \Delta(\theta)$,
\begin{align*}
0 &= \dot b + \alpha \dot \Delta \\
\alpha &= \frac{- \dot b}{\dot \Delta}
\end{align*}
Thus, $g^{-1}(\theta) = \frac{-\dot b}{\dot \Delta}$ is convex, decreasing and $g(\alpha)$ is concave, increasing.  Finally $a(g(\alpha))$ is convex, decreasing by the composition of a convex, non-increasing and concave function.  
\end{proof}

The proof for \cref{thm:convex} follows.
\begin{proof}
\textbf{Direction $\mathcal{T}$ to $\mathcal{V}$}
Holds by \cref{thm:monotone}, as linearity implies convexity.

\textbf{Direction $\mathcal{V}$ to $\mathcal{T}$}
Starting from the definition of $e$ and using the notational shorthand $\ell_P(f) = \mathbb{E}_{x, y \sim P}\ell(f(x), y)$,
\begin{align}
e_i &= \mathbb{E}_{x, y \sim P_\mathcal{V}} \ell(f_i(x), y) \nonumber \\
 &= \ell_{P_\mathcal{V}}(f_i) \nonumber \\
&= \ell_{P_\mathcal{V}}(\argmin_{f \in \mathcal{F}} \mathbb{E}_{x, y \sim M_{P_\mathcal{T}, P_\mathcal{V}}(1-\frac{i}{n}, \frac{i}{n})} \ell(f(x), y)) \nonumber \\
&= \ell_{P_\mathcal{V}}\left(\argmin_{f \in \mathcal{F}} \frac{i}{n} \ell_{P_\mathcal{V}}(f) + \left(1 - \frac{i}{n}\right) \ell_{P_\mathcal{T}}(f) \right)  \nonumber \\
\end{align}
By \cref{lem:convex}, $e$ is convex with respect to $\frac{i}{n}$, and thus also with respect to $i$ since $n$ is a fixed constant.
\end{proof}

\subsection{Proof of \cref{thm:basis}} \label{app:basis_proof}
\begin{proof}
\begin{equation*}
(A\Psi)_{ij} = \underset{k_{n} \sim \textup{Binomial}(n, p_i)}{\mathbb{E}} \psi_j\left(\frac{k_{n}}{n}\right)
\end{equation*}
By the weak law of large numbers, $\frac{k_{n}}{n} \overset{p}{\to} p_i$.  Further, $\psi_j\left(\frac{k_{n}}{n}\right) \overset{p}{\to} \psi_j(p_i)$ by the continuous mapping theorem.  Finally, $\mathbb{E}\psi_j\left(\frac{k_{n}}{n}\right) \rightsquigarrow \mathbb{E}\psi_j(p_i) = \psi_j(p_i)$ by the Portmanteau lemma.  The matrix formed by $\psi_j(p_i)$ is invertible by the Unisolvence theorem when $p_0, \dotsc, p_s$ are unique.  
\end{proof}

\subsection{Proof of \cref{thm:sketch}} \label{app:sketch_proof}
\begin{proof}
Let $e^\prime$ be the solution to the sketched system $Se^\prime = b$.  Then
\begin{align*}
Se^\prime &= Ae = b \\
e^\prime &= S^{-1}Ae \\
e^\prime_0 &= (S^{-1}A)_{00}e_0 + \sum_{i=1}^{n}(S^{-1}A)_{0i}e_i \\
e^\prime_0 - e_0 &= \sum_{i=1}^{n}(S^{-1}A)_{0i}e_i.
\end{align*}
Let $s^\prime$ be the first row of $S^{-1}$.  By the Cauchy-Schwarz inequality, for all $i \geq 1$
\begin{align*}
|(S^{-1}A)_{0i}| &= |s^\prime \cdot A_i| \\
&= |s^\prime \cdot A_i - s^\prime \cdot S_{r(i)}| \\
&\leq \Vert s^\prime\Vert \Vert A_i - S_{r(i)})\Vert \\
&= \epsilon \Vert s^\prime\Vert.
\end{align*}
Finally, by \cref{thm:monotone}
\begin{align*}
|e_o^\prime - e_0| &\leq \sum_{i=1}^{n}|(S^{-1}A)_{0i}|e_i \\
&\leq \sum_{i=1}^{n}\epsilon \Vert s^\prime \Vert e_i \\
&\leq \epsilon n \Vert s^\prime \Vert e_0.
\end{align*}
\end{proof}

\section{Connections to the Bernstein basis and B\'ezier curves} \label{app:bernstein}
\subsection{Bernstein basis}
Recall that a Bernstein basis of degree $n$ is defined as
\begin{equation}
b_{j, n}(x) = {n \choose j}x^j(1-x)^{n-j} \quad j = 0, \dotsc, n
\end{equation}
and that this forms a basis for polynomials at most degree $n$. Then the Bernstein polynomial is defined as
\begin{equation}
    B_n(x) = \sum_{j=0}^n \beta_j b_{j, n}(x)
\end{equation}
where $B_j$ are the Bernstein coefficients. The B3 estimator $b_i = \sum_{j=0}^n e_jA_{ij}$ is equivalent to solving for the Bernstein coefficients $e_j = \beta_j$, where the Berstein basis is $A_{ij} = b_{j,n}(p_i)$.

\subsection{B\'ezier curves}
B\'ezier curves are closely related to Bernstein polynomials, using slightly different notation
\begin{align}
    B(t) &= \sum_{j=0}^n {n \choose j} t^j (1-t)^{n-j} \mathbf{P}_j \\
    &= \sum_{j=0}^n b_{j,n}(t) \mathbf{P}_j
\end{align}
where $\mathbf{P}_j$ are the B\'ezier control points. Once again, $A_{ij}$ from the B3 estimator is equivalent to the Bernstein basis function $b_{j,n}(p_i)$, and we solve for the B\'ezier control points $\mathbf{P}_0, \dotsc, \mathbf{P}_n $

\section{Additional Experiments} \label{app:additional}
\subsection{Additions to \cref{fig:fullfit}}
Due to space limitations, the Dota 2 and Parkinson's experiments were excluded from \cref{fig:fullfit}. We include them here in \cref{fig:fullfit_app} to demonstrate the theoretical properties in \cref{sec:properties} held across all datasets. The Dota 2 experiments had higher variance than others, though the monotonic and convex trend appears to hold true. 

\begin{figure*}[h]
    \begin{subfigure}[t]{0.495\textwidth}
        \centering
        \includegraphics[width=\textwidth]{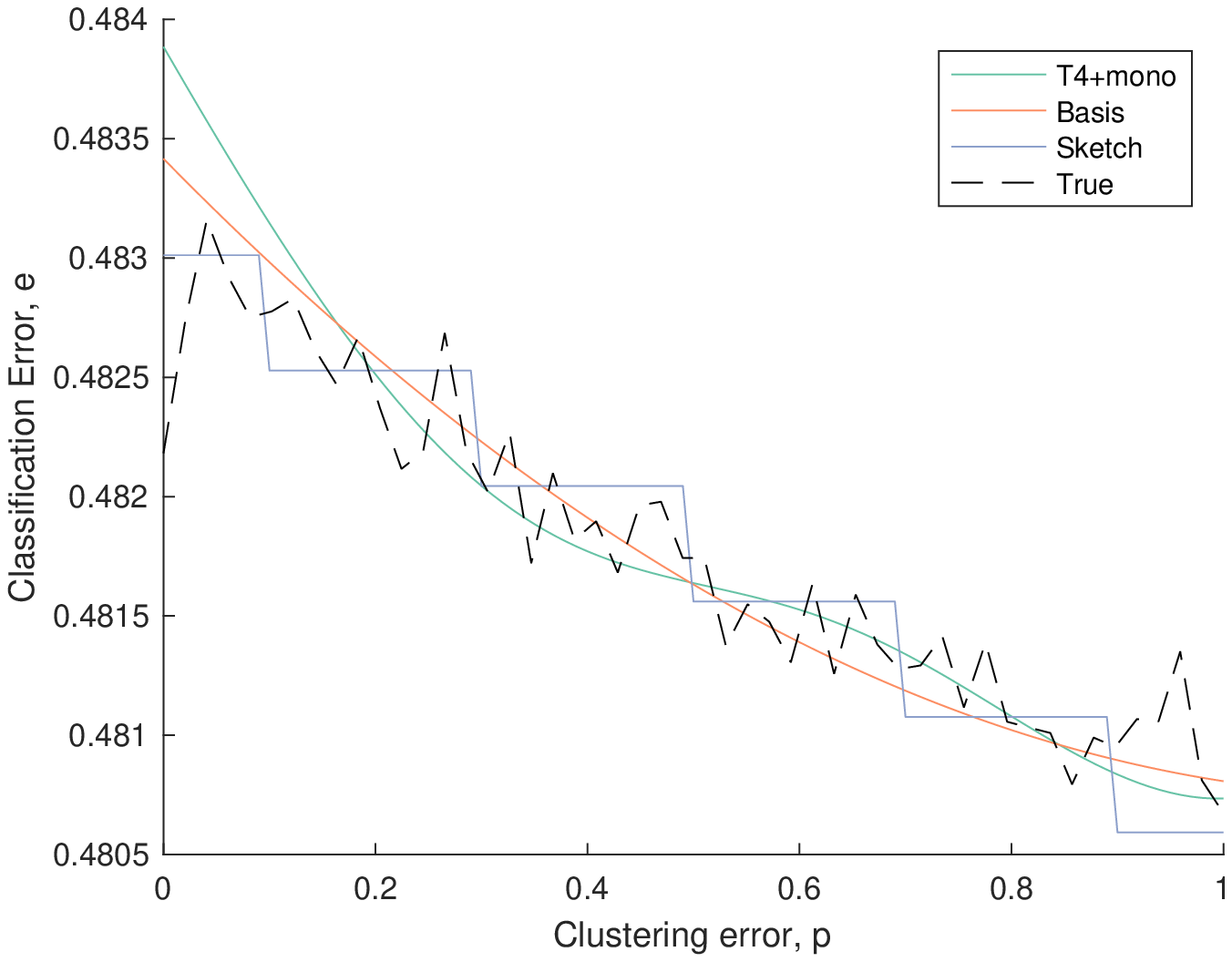}
        \caption{Dota 2}
    \end{subfigure}
    \begin{subfigure}[t]{0.495\textwidth}
        \centering
        \includegraphics[width=\textwidth]{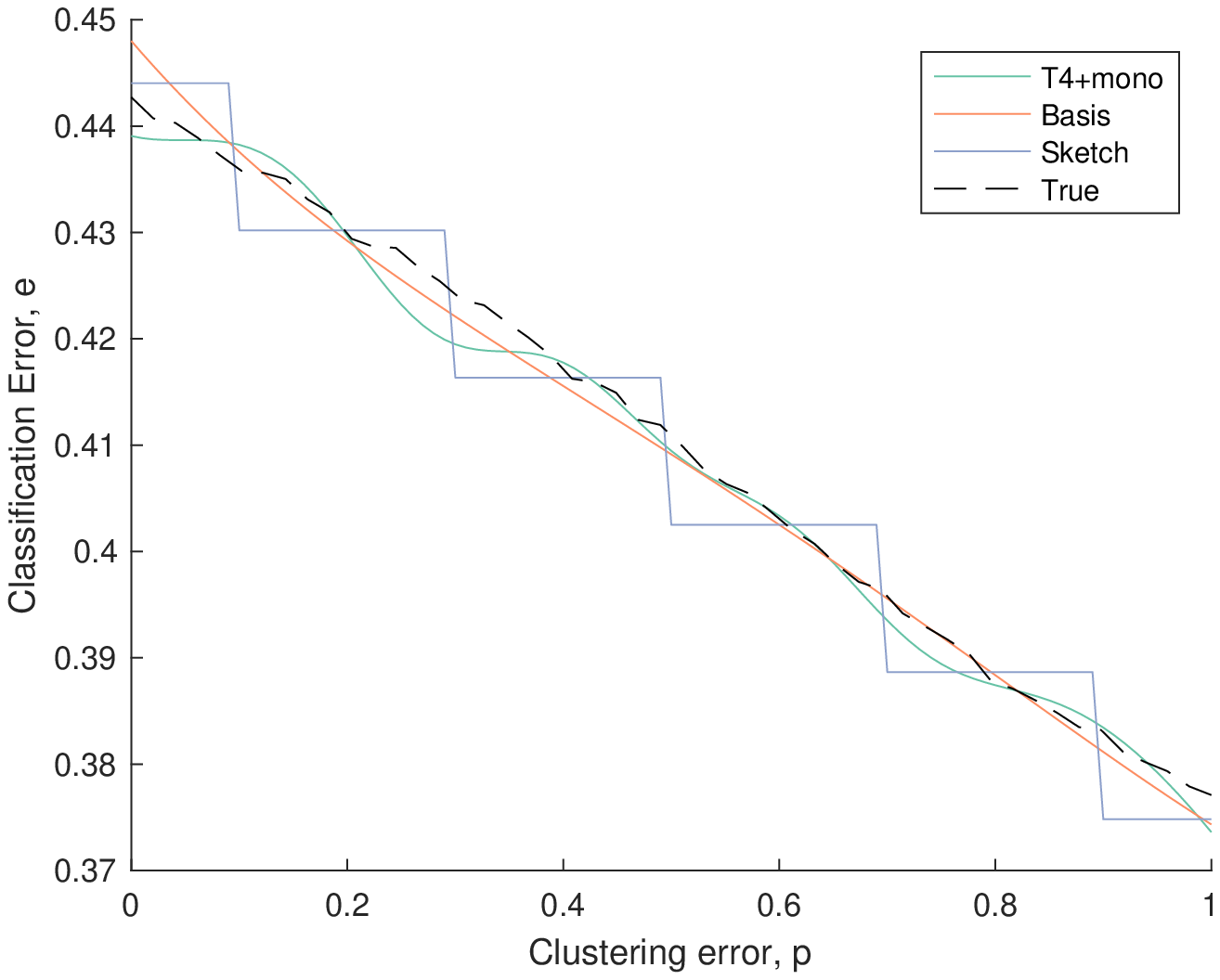}
        \caption{Parkinson's}
    \end{subfigure}
    \caption{Additions to \cref{fig:fullfit}. Empirical results show the loss is indeed convex and monotonically decreasing, validating our theoretical results in \cref{sec:properties}.  Note our methods are able to recover the full loss in addition to the true OOC loss $e_0$.}
    \label{fig:fullfit_app}
\end{figure*}

\subsection{Synthetic Experiments}
As an initial exploratory experiment, we generated synthetic data according to a partition model with $k=2$ parts, $m=2n$ rows in $A$ (i.e.  levels of dependency leakage) and initial dependency leakage probability $p_0=0.1$.  Our learner is a linear regression model with mean squared error loss.  Our exploratory synthetic results, presented in the Figure \ref{fig:synthetic} box plot, demonstrate that the basis function and matrix sketching estimators perform comparably or better than baseline methods.

\begin{figure*}[h]
        \centering
        \includegraphics[width=0.495\textwidth]{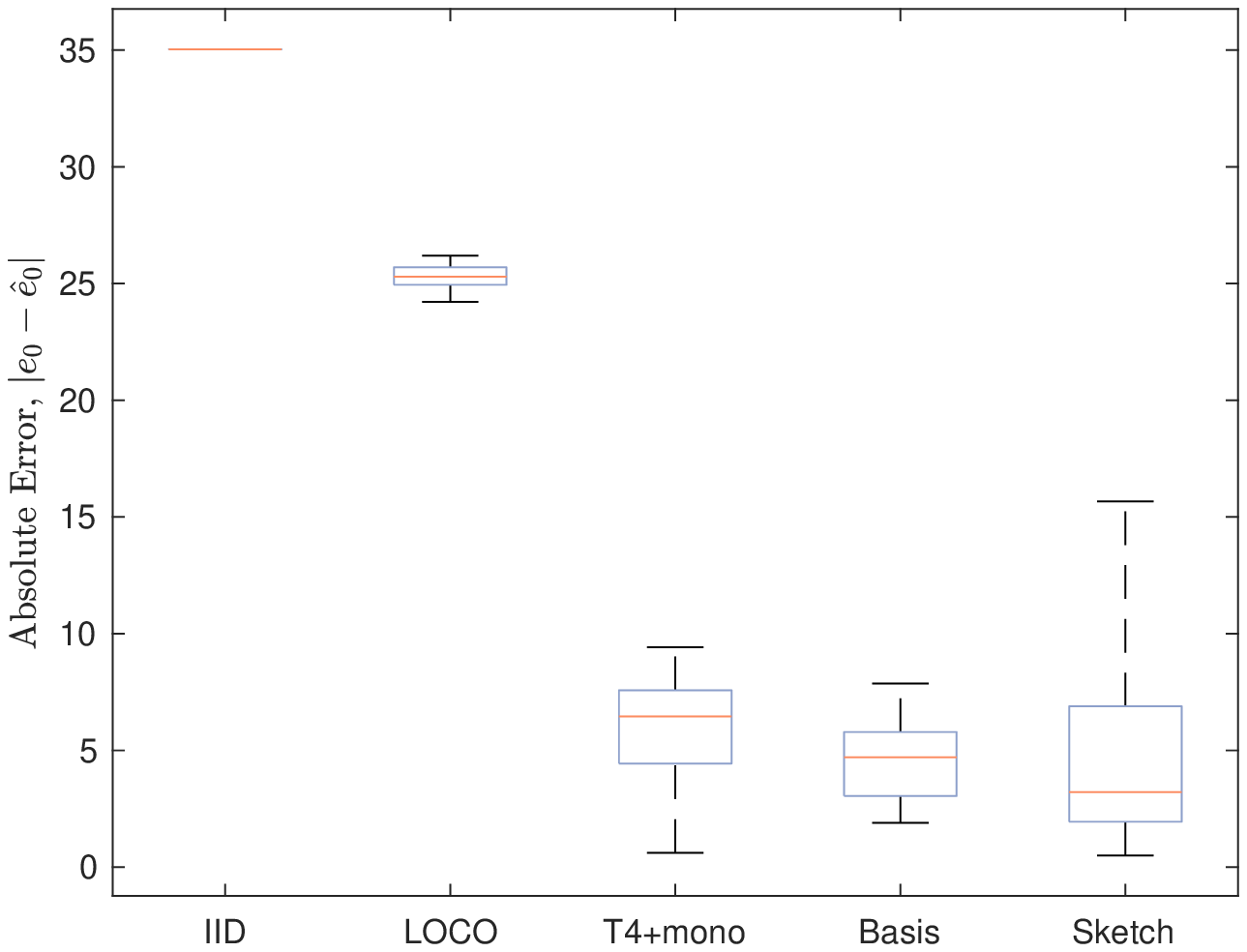}
        \includegraphics[width=0.495\textwidth]{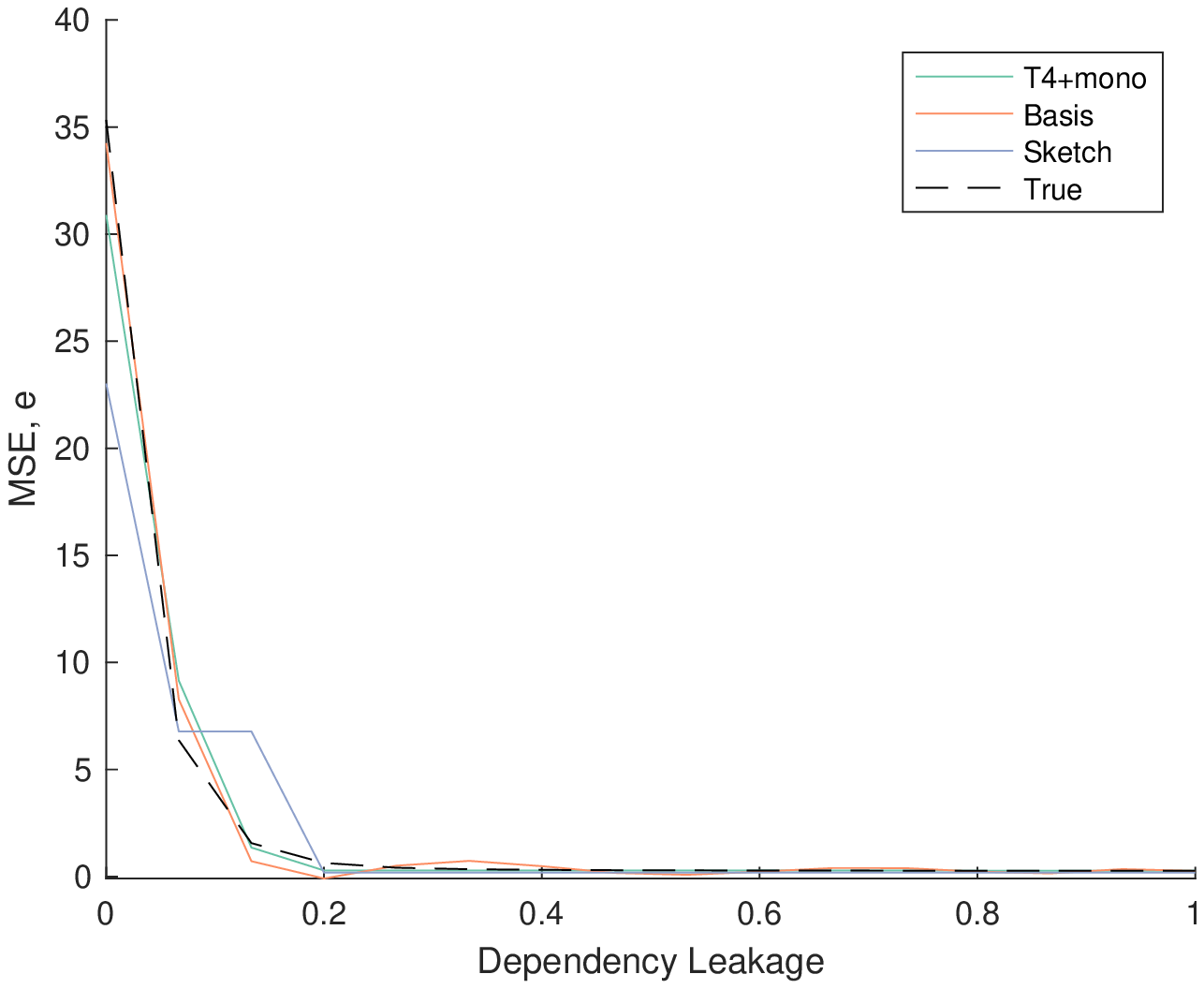}
    \caption{Synthetic regression results.}
    \label{fig:synthetic}
\end{figure*}

\section{Experimental Details} \label{app:details}
In the sketching approximation, we formed $k$ nearly equally sized groups of adjacent columns from $A$ when forming the sketched matrix $S$. Even after sketching, we found it beneficial to add some regularization comparable to T4+mono, referred to as $\lambda_s$ (the regularization used in T4+mono is referred to as $\lambda_{\textup{T4}}$). We found that other approaches, including using $k$-medoids to group the columns of $A$, did not provide any benefits and were more complicated. In all experiments we set $k=7$.

In the basis function approximation, we found that using simple, low-order polynomials was sufficient. Higher order polynomials tended to be unstable. After observing $b$, we chose to use either a 2nd or 7th order polynomial, depending on the curvature of $b$.

The whisker plots in \cref{fig:boxplots} are generated over 10 independent trials, where the whiskers correspond to most extreme values over those trials (i.e.\ no outliers removed).

The complete set of experimental parameters are shown in \cref{tab:params}. We made an effort to limit fitting to a specific dataset, and kept most parameters the same across all experiments. In the Dota 2 experiments, the availability of sufficient training data allowed us to increase $|\mathcal{T}|$ to 1000. Further, after completing the Heart and Census experiments, we reduced the number of rows $m$ in $A$ by an order of magnitude to speed up experimentation, and correspondingly increased the regularization $\lambda$.


\begin{sidewaystable}
  \caption[Dataset details and hyperparameters used in all experiments.]{Parameters used in all experiments. $n$ is the number of samples in the training set, $|\mathcal{V}|$ is the number of samples in the validation set, $t$ is the number of resamples in \cref{alg:b3}, $\lambda$'s are the regularization strengths in the T4+mono and sketching method, $m$ is the number of corruption levels (i.e.\ the number of rows in $A$), $k$ is the number of sketching groups and $d$ is the number of features in the dataset.}
  \centering
  \begin{tabular}{p{1.5cm}ccccccccccp{1.5cm}p{2cm}p{2cm}p{5cm}}
    \toprule
    & \multicolumn{14}{c}{Parameter}                   \\
    \cmidrule(r){2-15}
    Dataset     & $n$ & $|\mathcal{T}|$ & $|\mathcal{V}|$ & $d$ & t & $\lambda_{\textup{T4}}$ & $\lambda_s$ & $s$ & $m$ & $k$ & Latent cluster & Training clusters & Validation clusters & Features \\
    \midrule
     Synthetic & $\infty$ & 15 & 1000 & 2 & 1000 & 0.1 & 0.01 & 7 & 30 & 10 & - & - & - & - \\
     Heart\footnote{https://archive.ics.uci.edu/ml/datasets/heart+Disease} & 100 & 100 & 100 & 12 & 1000 & 10 & 0.1 & 7 & 200 & 7 & Location & Cleveland, VA, Switzerland & Hungary & age, sex, cp, trestbps, chol, fbs, restecg, thalach, exang, oldpeak, slope, thal \\
     1994 US Census\footnote{https://archive.ics.uci.edu/ml/datasets/adult} & 100 & 100 & 100 & 5 & 10000 & 10 & 0.1 & 7 & 200 & 7 & Native country & United States, El Salvador, Germany, Mexico, Philippines, Puerto Rico & India, Canada & age, education\_num, hours\_per\_week, race, occupation \\
     Parkinson\footnote{https://archive.ics.uci.edu/ml/datasets/Parkinson+Speech+Dataset+with++Multiple+Types+of+Sound+Recordings} & 100 & 100 & 100 & 26 & 10000 & 1000 & 0.1 & 2 & 20 & 7 & Subject & $2, 3, 4,$ $6, 7, 8, \dotsc$ & $1, 5, 9, \dotsc$ & jitter\_local, jitter\_abs, jitter\_rap, jitter\_ppq5, jitter\_ddp, shimmer\_local, shimmer\_db, shimmer\_apq3, shimmer\_apq5, shimmer\_apq11, shimmer\_dda, ac, nth, htn, median\_pitch, mean\_pitch, std\_dev, min\_pitch, max\_pitch, pulses, periods, mean\_period, std\_dev\_period, unvoiced, breaks, deg\_breaks \\
     Dota 2\footnote{https://archive.ics.uci.edu/ml/datasets/Dota2+Games+Results} & 100 & 100 & 100 & 114 & 1000 & 1000 & 0.1 & 2 & 20 & 7 & Type & 1, 2 & 3 & hero0, hero1, $\dotsc$, hero112 \\
    \bottomrule
  \end{tabular}
  \label{tab:params}
\end{sidewaystable}

\end{document}